\documentclass{article}
\usepackage{iclr2027_conference,times}

\usepackage{amsmath,amssymb,booktabs,microtype,graphicx,bm}
\usepackage{xcolor}
\usepackage[colorlinks=true,linkcolor={red!60!black},citecolor={blue!60!black},urlcolor={blue!60!black}]{hyperref}
\usepackage{url}
\usepackage{float}
\usepackage{amsthm}
\usepackage{enumitem}
\usepackage{placeins}

\newcommand{\Var}{\operatorname{Var}}
\newcommand{\Cov}{\operatorname{Cov}}
\newcommand{\E}{\mathbb{E}}
\newtheorem{proposition}{Proposition}
\newtheorem{corollary}{Corollary}

\makeatletter
\def\thanks#1{\protected@xdef\@thanks{\@thanks
        \protect\footnotetext[\c@footnote]{#1}}}
\makeatother

\title{Unit-level aggregation breaks coverage \\ in causal foundation models}

\author{Gabe Schulman \\ New York University \\ \texttt{gs4342@nyu.edu} \And Lucius E.J. Bynum\textsuperscript{\dag}  \\ New York University \\ \texttt{lucius@nyu.edu} \And Yoav Wald\textsuperscript{\dag} \\ Technion \\ \texttt{yoav.wald@technion.ac.il}
\thanks{\textsuperscript{\dag}Equal supervision.}
}

\iclrfinalcopy

\begin{document}
\maketitle
\lhead{}\renewcommand{\headrulewidth}{0pt}

\begin{abstract}
Causal foundation models (CFMs) show promising results for treatment-effect estimation, with unit-level estimates that are competitive on standard benchmarks. Many causal analyses, however, inform decisions through effects averaged over groups of units, such as the full sample, the treated units, or a subgroup, and for such decisions a CFM is useful only if its credible intervals are calibrated. We examine these intervals and find that their coverage degrades, both in simple synthetic settings and on standard benchmarks. On ACIC 2016, for example, temperature-calibrated 95\% intervals from CausalPFN \citep{causalpfn2025} fall below nominal coverage beyond groups of 30 out of 4802 units, and intervals for the full-sample average contain the true value in only 20\% of datasets.

We trace this problem to how these models are trained and how their intervals are formed. They output a predictive distribution for one unit at a time and, at inference time, combine the unit-level variances as if the errors on different units were independent. We show formally that posterior uncertainty over the data-generating process makes these errors correlated. As a result, the combined variance is too small by a factor that grows with the size of the group. Empirically, this factor predicts where coverage fails and by how much. We then show that two neural process models, which make joint predictions for a group of units, account for this correlation and substantially improve coverage on synthetic tasks and, as heads on a frozen CausalPFN, on all six benchmarks.
\end{abstract}

\section{Introduction}
\label{sec:intro}

Causal Foundation Models (CFMs) are pre-trained on simulated causal datasets and estimate treatment effects for a new dataset in a single forward pass. Recent examples, including CausalPFN \citep{causalpfn2025}, Do-PFN \citep{dopfn2025}, and CausalFM \citep{causalfm2026}, perform competitively on synthetic and semi-synthetic benchmarks without fitting a new estimator to each dataset. CFMs typically estimate effects at the unit level, through Conditional Average Treatment Effects (CATEs) that describe the average treatment effect as a function of the observed covariates.

Many causal questions, however, concern effects averaged over a group of units, such as the population average treatment effect (PATE), the average treatment effect on the treated (ATT), and subgroup effects. For instance, clinical trials often focus on subgroup estimands \citep{ich2019, ema2019}, drugs have been approved for specific subpopulations after subgroup findings motivated dedicated trials \citep{taylor2004}, and large policy evaluations report prespecified subgroup analyses alongside intervals too wide to rule out clinically meaningful effects \citep{baicker2013}. We call any average of the CATE over a set of units a \emph{subset-CATE} \citep[see, e.g.,][]{e24121782} and define it formally in Section~\ref{sec:setup}. In principle, a CFM can estimate any subset-CATE by averaging its unit-level estimates. One trained model can then target any subgroup without refitting, including the full sample and the treated units, which give analogues of the PATE and ATT.

In applied settings such as medicine, decisions often depend on how uncertain the estimate is, so a CFM is useful for these questions only if its intervals for group averages are calibrated. We show that the intervals current CFMs report for group averages, including the full sample, have no coverage guarantee and fall below nominal on every benchmark we test, and we propose two models that address this problem. Some CFMs output a predictive distribution for each unit \citep{causalpfn2025, causalfm2026}, and the unit-level credible intervals CausalPFN forms from it are conservative on our benchmarks. CausalPFN also returns a full-sample average, which it reports as the ATE, with a 95\% credible interval formed by averaging independent draws from the unit-level predictive distributions after an optional temperature calibration. On ACIC 2016 \citep{dorie2019}, however, this interval contains the true full-sample average in only 8 of the 40 datasets we use when calibrated and in 18 of 40 when not. \citet{mourao2026} report similarly poor ATE coverage on simulated clinical data. Intervals from the same procedure fall below nominal coverage for subsets larger than 11\% of the sample on IHDP \citep{hill2011, shalit2017} and 0.6\% on ACIC 2016 (Figure~\ref{fig:benchmarks}).

\begin{figure}[t]
\centering
\includegraphics[width=0.7\textwidth]{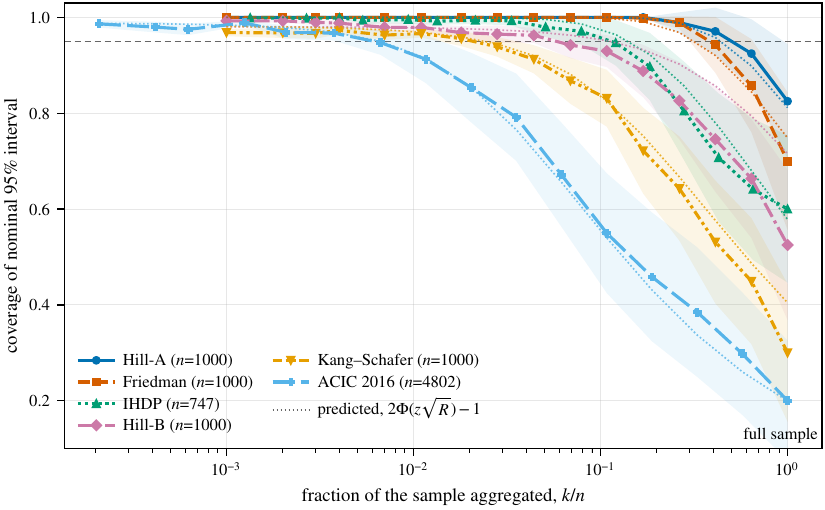}
\caption{Coverage of nominal $95\%$ subset-CATE intervals from temperature-calibrated CausalPFN on standard benchmarks, with unit-level aggregation. Subsets are drawn uniformly at random. The horizontal axis is the fraction of the sample aggregated, and its right edge is the full sample. Thin dotted lines are the prediction of Equation~\eqref{eq:intro-coverage}, with $m_S$ and $\rho_S$ measured once per benchmark and held fixed across $k$, and the horizontal dashed line marks nominal coverage.}
\label{fig:benchmarks}
\end{figure}

Current CFMs form the aggregate interval by summarizing each unit's uncertainty separately and combining the summaries as if the units' errors were independent; we call this \emph{unit-level aggregation}. Because Do-PFN \citep{dopfn2025}, CausalFM \citep{causalfm2026}, and CCPFN \citep{ccpfn2026} provide no aggregate estimator, any interval for a subset-CATE has to be built from their unit-level outputs. When there is uncertainty about the data-generating process, however, a model that overestimates one unit's effect tends to overestimate the others' too. Averaging over units then shrinks only the part of the error that is not shared. For a subset of $k$ units, let $\rho_S$ be the correlation between the model's errors on two distinct units of the subset, taken over tasks (draws of the SCM, the dataset, and the subset; Section~\ref{sec:design}), and let $m_S$ be the ratio of the variance the model reports for a unit to its mean squared error ($m_S = 1$ for calibrated unit-level intervals). Under mean-zero Gaussian errors and a fixed reported variance, a nominal $95\%$ interval formed by unit-level aggregation has coverage
\begin{equation}
\mathcal{C}(k) \;=\; 2\,\Phi\!\left(1.96\,\sqrt{\frac{m_S}{1+(k-1)\rho_S}}\,\right) - 1 .
\label{eq:intro-coverage}
\end{equation}
The denominator is the design effect of survey sampling \citep{kish1965}, applied here to estimation errors. Because it grows linearly in $k$, coverage tends to zero for any fixed $\rho_S > 0$. With over-wide unit-level intervals ($m_S > 1$), coverage still falls, but at larger $k$. On ACIC 2016, for example, CausalPFN has $m_S = 1.63$ and $\rho_S = 0.02$. For these values, Equation~\eqref{eq:intro-coverage} predicts that coverage falls below nominal at 32 units and reaches $20\%$ on the full sample; we observe 30 units and $20\%$.

Once the data are observed, the remaining uncertainty in a unit's effect comes from the structural causal model (SCM) that generated them. The effects of all units depend on that same SCM. Unless the data determine the SCM, the unit effects are dependent a posteriori, and an estimator's errors on different units are in general correlated. The exact posterior represents this dependence. Unit-level aggregation ignores it, even though the model's unit-level intervals reflect the same uncertainty.

Accounting for this correlation requires sampling the unit effects jointly. Let $\tau_i$ be the CATE of unit $i$, $S$ the set of units in the subset, $\zeta$ the parameters of the SCM, and $\pi(\zeta \mid \mathcal{D})$ their posterior given the dataset $\mathcal{D}$. Joint and unit-level sampling both average draws of $(\tau_i)_{i\in S}$ and differ only in where the integral over $\zeta$ is placed. They draw from
\begin{align}
\text{unit-level:}\quad & \prod_{i\in S}\int p(\tau_i \mid \zeta)\, d\pi(\zeta \mid \mathcal{D}), \label{eq:intro-unit}\\
\text{SCM-marginalized:}\quad & \int \prod_{i\in S} p(\tau_i \mid \zeta)\, d\pi(\zeta \mid \mathcal{D}). \label{eq:intro-scm}
\end{align}
Placing the integral inside the product, as unit-level aggregation does, treats the unit effects as independent once the SCM is integrated out. They share the SCM, however, and this induces the error correlation $\rho_S$ in Equation~\eqref{eq:intro-coverage}. \emph{SCM-marginalized aggregation} places the integral outside the product and so accounts for the dependence between units. We approximate it with a latent neural process (NP) \citep{garnelo2018np}, which replaces the SCM with a learned latent variable, and with a direct-decode conditional neural process (CNP) \citep{garnelo2018cnp}, which learns the distribution of the subset-CATE directly.

Our contributions are as follows:
\begin{itemize}[leftmargin=*, itemsep=1pt, topsep=2pt]
\item \textbf{Diagnosis.} We show that unit-level aggregation understates the variance of any subset-CATE, whether over the full sample, the treated units, or a subgroup, by the design effect. Under Gaussian errors with a fixed reported variance, its coverage is given by Equation~\eqref{eq:intro-coverage} (Proposition~\ref{prop:design}).
\item \textbf{Source.} We show that, for the posterior mean, the error covariance is the covariance of unit effects under the posterior over the SCM (Proposition~\ref{prop:source}). It is generally nonzero unless the posterior determines the queried unit effects, and it cannot be computed from unit-level marginals alone.
\item \textbf{Method.} We propose two SCM-marginalized models, a latent NP and a direct-decode CNP, that integrate over the SCM jointly for all units of the subset (Section~\ref{sec:fix}), and train them both on a synthetic family and as heads on a frozen CausalPFN.
\item \textbf{Evidence.} CausalPFN's intervals lose nominal coverage on all six standard benchmarks, including for the full sample. With the weights held fixed, unit-level coverage at the largest subset is at most $0.47$ across $60$ cells, while SCM-marginalized coverage has a median of $0.89$. On the six benchmarks, the two models as heads on a frozen CausalPFN cover the full-sample average in $90\%$ and $89\%$ of $240$ datasets, against $25\%$ and $52\%$ for unit-level aggregation from the same weights and $53\%$ for CausalPFN. Proposition~\ref{prop:design} predicts coverage on the synthetic family to within a median absolute error of $0.05$ (Section~\ref{sec:experiments}).
\end{itemize}

\section{Related work}
\label{sec:related}

\textbf{Inference for aggregate estimands.} Uncertainty quantification for group effects, the problem we describe for CFMs in Section~\ref{sec:intro}, has been studied for non-amortized estimators. These can target group effects while retaining the dependence among units, e.g., by averaging within each posterior draw \citep{hill2011}, as our latent NP does, or by building the aggregate from orthogonal scores \citep{chernozhukov2018, semenova2021, chernozhukov2018gates}. Conformal methods give distribution-free coverage for individual treatment effects \citep{lei2021, alaa2023, memmesheimer2025}, but target a unit's outcome, not a group average. \\ [2pt]
\textbf{Design effects and clustered data.} Correlation within groups inflates the variance of an average by the design effect, which \citet{kish1965} introduced for complex surveys and which also applies to cluster sampling \citep{cochran1977} and cluster-randomized trials \citep{hedges2007}. We apply it to a model's estimation errors, whose correlation comes from the posterior over the SCM, not from the sampling design. \\ [2pt]
\textbf{Uncertainty quantification for amortized inference.} In the amortized prediction setting, prior-data fitted networks (PFNs) are trained on synthetic learning problems drawn from a prior, so that a single forward pass approximates posterior inference \citep{muller2022}, and they are state of the art on tabular data \citep{hollmann2023, hollmann2025}. \citet{nagler2023} studies the statistical properties of PFNs as pre-tuned predictors, analyzing their bias and variance. Later work extends PFNs to full posteriors \citep{reuter2025} and to credible sets for a global functional \citep{ng2025}. Notably, \citet{naglerrugamer2025} note that their pointwise procedure does not provide joint inference across query points, which a subset-CATE interval requires, and leave it to future work. Equations~\eqref{eq:intro-unit} and~\eqref{eq:intro-scm} correspond to the marginal and joint predictive distributions of \citet{wen2021predictions} and \citet{osband2023enn}. Gaussian and autoregressive CNPs model the dependence between target points that standard CNPs omit \citep{markou2022practical, bruinsma2023autoregressive}, and quantile CNPs have been applied to amortized estimation of unit-level and population-level causal estimands \citep{bynum2026computational}. \\ [2pt]
\textbf{Causal foundation models.} PFNs and other amortized methods have been applied to causal effect estimation \citep{nilforoshan2023zero, zhang2023towards, bynum2025black, causalpfn2025, dopfn2025, causalfm2026, ccpfn2026}, where they perform competitively on standard benchmarks. Those that report uncertainty typically do so for single units or, like CausalPFN, aggregate independent unit-level draws, and two recent papers examine their uncertainty for aggregate effects. \citet{mourao2026} evaluate CausalPFN on simulated clinical scenarios and report poor coverage of its $95\%$ credible interval for the ATE, attributing it to estimation bias and inadequate uncertainty quantification. \citet{melnychuk2026} show that a PFN's posterior for an aggregate effect can have a prior-induced bias that does not vanish asymptotically, which precludes frequentist consistency. They also note that a PFN's pointwise predictive distributions do not determine a unique joint posterior and that sampling them independently ignores the dependence between points. They obtain joint draws with martingale posteriors. We quantify the coverage lost by independent aggregation for arbitrary subsets and show that CausalPFN forms its full-sample interval this way. Our latent NP models this dependence through a shared latent variable, and our CNP predicts the subset average directly.

\section{Setup}
\label{sec:setup}

We consider effect estimation under a binary treatment. Each unit $i$ has covariates $X_i \in \mathcal{X}$, a treatment $T_i \in \{0,1\}$, and potential outcomes $Y_i(0), Y_i(1)$, one of which is observed as $Y_i = Y_i(T_i)$. We write the CATE as $\tau(x)=\E[Y(1)-Y(0) \mid X=x]$ and $\tau_i := \tau(X_i)$ for its value at unit $i$.

\textbf{Subset-CATE.} Our target averages these unit-level effects over a queried subset of units. Let $S\subseteq\{1,\dots,n\}$ be the indices of the subset and $k=|S|$ its size. We define
\begin{equation}
\theta_S \;:=\; \frac{1}{k}\sum_{i\in S} \tau_i ,
\label{eq:gate}
\end{equation}
and refer to $\theta_S$ as the subset-CATE.

Because $\theta_S$ averages $\tau(X_i)$ at the realized covariates of the units in $S$, it is a finite-sample estimand---the estimand of \citet{abadieimbens2006} restricted to the units in $S$. It differs from the sample average treatment effect $k^{-1}\sum_{i \in S}(Y_i(1)-Y_i(0))$, which includes unit-level noise, and from the PATE $\E[Y(1)-Y(0)]$.

At $k=1$, $\theta_S$ is the CATE of a single unit. At $S=\{1,\dots,n\}$ it is a sample analogue of the PATE, and the quantity existing CFMs report as the ATE; at $S=\{i:T_i=1\}$ it is the analogous quantity for the ATT.

\textbf{Meta-learning.} CFMs are amortized causal estimators, with recent methods building on the meta-learning tradition of prior-data fitted networks \citep{muller2022, hollmann2023}. Such models are trained on tasks $\mathcal{T}_1,\dots,\mathcal{T}_J$ drawn i.i.d.\ from a task distribution. Each task $\mathcal{T}_j=(\zeta_j,\mathcal{D}_j,S_j)$ consists of a structural causal model (SCM) $\zeta_j\sim\pi_0$, a dataset $\mathcal{D}_j=\{X_i,T_i,Y_i\}_{i=1}^{n_j}$ drawn from the SCM, and a subset $S_j$ of one or more units, drawn given $\mathcal{D}_j$, whose subset-CATE is queried. We write $\tau(x;\zeta)$ for the CATE under SCM $\zeta$ when the dependence matters. Current CFMs typically set $|S_j|=1$, so the training target is the CATE. The model then returns a predictive distribution for a single unit's treatment effect, whose mean and variance we write as $\hat{\tau}_i$ and $\hat{\sigma}_i^2$. Estimating $\theta_S$ for $|S|>1$ with a credible interval requires combining $\hat{\tau}_i$ and $\hat{\sigma}_i^2$ across units in $S$.

\textbf{Uncertainty quantification and coverage.} CFMs that provide uncertainty quantification typically report a credible interval, since they are designed to approximate a posterior distribution under a known prior. We denote the interval for $\theta_S$ as $[\hat{l}_{S},\hat{u}_{S}]$. When it is symmetric about $\hat{\theta}_S$ we write $[\hat{l}_{S},\hat{u}_{S}]=\hat{\theta}_{S}\pm z_{1-\alpha/2}\,\hat{\sigma}_{S}$, with $\hat{\sigma}_{S}$ the standard deviation the model reports for $\theta_S$ and $z_{1-\alpha/2}$ the standard normal quantile ($z_{0.975}\approx1.96$ at our $\alpha=0.05$).

We evaluate credible intervals by their coverage, the fraction of queries where the interval contains the true subset-CATE. Coverage at subset size $k$ is
\begin{equation}
\mathcal{C}(k) \;:=\; \Pr\big(\theta_{S}\in[\hat{l}_{S},\hat{u}_{S}]\big),
\label{eq:coverage-def}
\end{equation}
where the probability is over the SCM $\zeta\sim\pi$, the dataset $\mathcal{D}\mid\zeta$, and the subset $S$ of size $k$, drawn uniformly or by a selection rule. Here $\pi$ is the distribution of evaluation tasks, which can differ from a model's training prior $\pi_0$; for a benchmark, the probability is over its datasets. Appendix~\ref{app:protocol} gives the estimator we use for $\mathcal{C}(k)$ and the task counts behind each reported value.

\textbf{Identification.} Our analysis assumes that all confounders are observed and that positivity holds, so that the subset-CATE is identified. Our synthetic family satisfies both by construction (Appendix~\ref{app:dgps}). The benchmarks have no unobserved confounding, but IHDP and some ACIC 2016 settings have limited overlap, so part of their error can come from extrapolation.

\section{Why unit-level aggregation breaks coverage}
\label{sec:collapse}

\subsection{The design effect}
\label{sec:design}

Let $e_i := \hat\tau_i - \tau_i$ denote the estimator's error at unit $i$. Since $\theta_S$ is an average of unit-level effects, the subset-CATE error is the average of unit-level errors, $\bar e_S = \hat\theta_S - \theta_S = k^{-1}\sum_{i \in S} e_i$. We define two moments of the errors and their ratio, with expectations over tasks $(\zeta,\mathcal{D},S)$ as in Equation~\eqref{eq:coverage-def},
\begin{align}
  \sigma_S^2 &\;=\; \E\Big[\frac{1}{k}\sum_{i \in S}e_i^2\Big], &
  c_S &\;=\; \E\Big[\frac{1}{k(k-1)}\sum_{\substack{i,j \in S \\ i \neq j}} e_i e_j\Big], &
  \rho_S &\;=\; \frac{c_S}{\sigma_S^2}.
  \label{eq:moments}
\end{align}
We use raw second moments because coverage depends on the total error, including any bias in unit-level estimates. We define the moments over the queried subset because a query that selects units on a covariate (e.g., patients aged 40--60) may have more strongly correlated errors than the dataset as a whole. When $S$ is drawn uniformly, $\sigma_S^2$, $c_S$, and $m_S$ equal their values for the full sample and so do not depend on $k$. Appendix~\ref{app:subgroups} shows how $\rho_S$ and $m_S$ change for selected subgroups.

Existing CFMs combine the unit-level variances as if the errors were independent, which gives the reported variance
\begin{equation}
  \hat\sigma^2_{S} \;=\; \frac{1}{k^2}\sum_{i \in S}\hat\sigma_i^2 .
  \label{eq:sigma}
\end{equation}
CausalPFN's aggregate estimator averages independent draws from the units' predictive distributions, which makes it a Monte Carlo version of Equation~\eqref{eq:sigma} (Appendix~\ref{app:causalpfn-equiv}).\footnote{\path{src/causalpfn/causal_estimator.py:517} and \path{src/causalpfn/models/icl_model.py:61} of \citet{causalpfn_repo}, at commit \texttt{7da4afa}.}

\begin{proposition}[Design effect; cf.\ \citealt{kish1965}]
\label{prop:design}
Let $\hat\theta_S = k^{-1}\sum_{i\in S}\hat\tau_i$, with expectations over $\zeta$, $\mathcal{D}$, and $S$ as in Equation~\eqref{eq:coverage-def}. The mean squared error of $\hat\theta_S$ is
\begin{equation}
  \E\big[\bar e_S^2\big]
  \;=\; \frac{1}{k^2}\,\E\Big[\,\sum_{i \in S}e_i^2 \;+\; \sum_{\substack{i,j \in S \\ i \neq j}}e_i e_j\,\Big]
  \;=\; \frac{\sigma_S^2}{k}\Big[\,1 + (k-1)\rho_S\,\Big].
  \label{eq:design}
\end{equation}
If the interval is $\hat\theta_S \pm z_{1-\alpha/2}\hat\sigma_S$ with $\hat\sigma_S^2$ from Equation~\eqref{eq:sigma}, then the ratio of reported to actual squared error and (when $\hat\sigma_S$ does not vary across queries) the coverage are
\begin{align}
  R \;&:=\; \frac{\E[\hat\sigma^2_{S}]}{\E[\bar e_S^2]}
  \;=\; \frac{m_S}{1 + (k-1)\rho_S}, \label{eq:bias}\\
  \mathcal{C}(k) \;&=\; \Pr\!\left(\frac{|\bar e_S|}{\sqrt{\E[\bar e_S^2]}} \;\le\; z_{1-\alpha/2}\sqrt{R}\,\right), \label{eq:coverage}
\end{align}
where $m_S := \E[\bar{\hat\sigma}_S^2]/\sigma_S^2$ and $\bar{\hat\sigma}_S^2 = k^{-1}\sum_{i \in S}\hat\sigma_i^2$ is the average variance the model reports for a unit in $S$. If, in addition, $\bar e_S$ is mean-zero Gaussian, then $\mathcal{C}(k) = 2\Phi(z_{1-\alpha/2}\sqrt{R})-1$.
\end{proposition}

\emph{Proof sketch} (Appendix~\ref{app:proofs}). Expanding $\bar e_S^2$ gives $k$ diagonal terms, which contribute $\sigma_S^2$ on average, and $k(k-1)$ cross terms, which contribute $c_S$; the rest follows by dividing Equation~\eqref{eq:sigma} by Equation~\eqref{eq:design} and standardizing the error.

The standardized error in Equation~\eqref{eq:coverage} has second moment one regardless of $k$, $\rho_S$, and $m_S$. As a result, once the shape of its distribution is fixed, these quantities affect coverage only through $R$. Under Gaussian errors, the interval attains nominal coverage at $R=1$, and for $R<1$ coverage is below nominal and increasing in $R$. For $\rho_S > 0$, setting $R=1$ gives the largest subset for which the reported interval is wide enough,
\begin{equation}
  k^* \;=\; 1 + \frac{m_S - 1}{\rho_S}.
  \label{eq:kstar}
\end{equation}
Calibrated unit-level intervals ($m_S = 1$) give $k^* = 1$, so nominal coverage holds for single units only. Narrower ones ($m_S < 1$) attain nominal coverage at no subset size under Gaussian errors. Wider ones attain it up to $k^*$, with $k^*-1$ inversely proportional to $\rho_S$.

Equations~\eqref{eq:design} and~\eqref{eq:bias} hold for any error distribution. When the reported variance varies across queries, Equation~\eqref{eq:intro-coverage} is an approximation, since $R$ is a ratio of expectations (Appendix~\ref{app:proofs}); we test it in Section~\ref{sec:designeffect}.

\subsection{The source of the error correlation}
\label{sec:definetti}

For the posterior mean, the correlation $\rho_S$ in Proposition~\ref{prop:design} comes entirely from posterior uncertainty over the SCM. Write $Z_i = (X_i, T_i, Y_i)$ for the observed record of unit $i$. Given the SCM, the units of a task are i.i.d. The SCM itself is random, however, so the joint distribution of the units is a mixture over SCMs,
\begin{equation}
p(Z_1,\dots,Z_n) = \int \prod_{i=1}^{n} p(Z_i \mid \zeta)\, d\pi(\zeta),
\label{eq:definetti}
\end{equation}
which does not factorize across units. Because observing one unit is informative about $\zeta$, and hence about the others, the units are exchangeable but not independent. De Finetti's theorem \citep{definetti1937}, in the general form of \citet{hewittsavage1955}, shows that any infinitely exchangeable sequence admits such a representation for some random measure. For CFMs, that measure is the observational distribution of the SCM, and under the assumptions of Section~\ref{sec:setup} each $\tau(X_i;\zeta)$ is a functional of it. A CFM that is permutation invariant in its context treats its data as exchangeable.

\begin{proposition}[Source of the correlation]
\label{prop:source}
For any estimator whose outputs are functions of $\mathcal{D}$,
\begin{equation}
  \Cov\big(e_i, e_j \mid \mathcal{D}\big) \;=\; \Cov_{\zeta\sim\pi(\zeta\mid\mathcal{D})}\big(\tau(X_i;\zeta),\, \tau(X_j;\zeta)\big).
  \label{eq:cS-definetti}
\end{equation}
If $S$ is conditionally independent of the unit effects given $\mathcal{D}$, then $c_S$ is the average, over $\mathcal{D}$, $S$, and pairs of distinct units in $S$, of this covariance plus $b_ib_j$, where $b_i := \E[e_i \mid \mathcal{D}]$. For the posterior mean under $\pi$, $b_i = 0$.
\end{proposition}

\emph{Proof sketch} (Appendix~\ref{app:proofs}). Given $\mathcal{D}$, the outputs $\hat\tau_i$ are fixed; given $\zeta$ as well, each $\tau_i$ is fixed. Hence $\Cov(e_i, e_j \mid \zeta, \mathcal{D}) = 0$, and the law of total covariance leaves only the covariance of $\tau_i$ and $\tau_j$ across posterior draws of $\zeta$. Averaging $\E[e_ie_j \mid \mathcal{D}] = \Cov(e_i,e_j \mid \mathcal{D}) + b_ib_j$ over $\mathcal{D}$, $S$, and pairs gives $c_S$. The term $b_ib_j$ is positive when the estimator errs in the same direction at units $i$ and $j$.

Equation~\eqref{eq:intro-scm} includes this covariance, and Equation~\eqref{eq:intro-unit} sets it to zero. The unit-level marginals $p(\tau_i \mid \mathcal{D})$ do not determine it (Appendix~\ref{app:proofs}), so recovering it requires a joint model of the unit effects or, as in martingale posteriors \citep{ng2025, melnychuk2026}, repeated forward passes.

\subsection{SCM-marginalized aggregation}
\label{sec:fix}

The posterior variance of $\theta_S$ averages the covariance of Proposition~\ref{prop:source} over all $k^2$ pairs of units,
\begin{equation}
  \Var\big(\theta_S \mid \mathcal{D}\big) \;=\; \frac{1}{k^2}\sum_{i\in S}\sum_{j\in S}\Cov_{\zeta\sim\pi(\zeta\mid\mathcal{D})}\big(\tau(X_i;\zeta),\,\tau(X_j;\zeta)\big).
  \label{eq:postvar}
\end{equation}
In practice, both distributions in Equations~\eqref{eq:intro-unit} and~\eqref{eq:intro-scm} are approximated by Monte Carlo sampling. The two schemes differ in whether the unit effects in one draw are sampled independently or computed under one shared draw of the SCM. Let $\tilde\theta_S = k^{-1}\sum_{i\in S}\tilde\tau_i$ be one draw of the subset-CATE and $\tilde\zeta \sim \pi(\zeta \mid \mathcal{D})$ one draw of the SCM. Conditional on $\mathcal{D}$, the variance of one draw is
\begin{align}
  \text{unit-level:}\quad & \tilde\tau_i \overset{\text{ind.}}{\sim} p(\tau_i \mid \mathcal{D}),
  & \Var\big(\tilde\theta_S \mid \mathcal{D}\big) &= \tfrac{1}{k^2}\textstyle\sum_{i\in S}\Var(\tau_i \mid \mathcal{D}), \label{eq:unitdraw}\\
  \text{SCM-marginalized:}\quad & \tilde\tau_i = \tau(X_i;\tilde\zeta),
  & \Var\big(\tilde\theta_S \mid \mathcal{D}\big) &= \Var(\theta_S \mid \mathcal{D}). \label{eq:scmdraw}
\end{align}
The first is what CausalPFN's aggregate estimator computes, and it omits the cross terms of Equation~\eqref{eq:postvar}. The second includes them, because the units in a draw share one SCM. For the exact posterior under $\pi$, the first gives $R = 1/\big(1+(k-1)\rho_S\big)$ in Proposition~\ref{prop:design}, and the second gives $R = 1$ and exact coverage at every $k$ (Corollary~\ref{cor:exact}). Since no CFM outputs draws of $\zeta$, the models below approximate SCM-marginalized aggregation without them.

\paragraph{Latent NP.} A latent NP \citep{garnelo2018np} replaces the SCM with a learned latent variable $z$. An encoder maps the context to a distribution $q_\psi(z \mid \mathcal{D})$, and a decoder maps $z$, the context, and a unit's covariates to a distribution over that unit's effect. The model's joint distribution over the unit effects is then
\begin{equation}
  p_\phi\big((\tau_i)_{i\in S} \mid \mathcal{D}\big) \;=\; \int \prod_{i\in S} p_\phi(\tau_i \mid z, X_i, \mathcal{D})\, q_\psi(z \mid \mathcal{D})\, dz ,
  \label{eq:latentnp}
\end{equation}
which has the SCM-marginalized form of Equation~\eqref{eq:intro-scm} with $z$ in place of $\zeta$. The model is trained on two likelihood terms estimated with draws of $z$ from $q_\psi(z \mid \mathcal{D})$, with no KL term \citep[cf.][]{foong2020convnp}: one for each unit's effect and one for $\theta_S$ on a random subset $S$, averaging the decoded unit effects within each draw (Appendix~\ref{app:protocol}). Because the decoder treats units as independent given $z$, any dependence between units that the model learns must be captured by $z$. The subset-level term trains this dependence directly, since for large subsets the variance the model assigns to $\theta_S$ comes mainly from how the within-draw averages vary across draws. At inference, we sample one $z$ per Monte Carlo draw, decode every unit of $S$ from it, and average, which implements Equation~\eqref{eq:scmdraw} with $z$ in place of $\zeta$. Uncertainty in the per-unit decoder given $z$ is still combined as in unit-level aggregation.

\paragraph{Direct-decode CNP.} The CNP \citep{garnelo2018cnp} takes $S$ as a mask in the forward pass and outputs a Gaussian $q_\phi(\theta_S \mid \mathcal{D}, S)$. When it is trained to minimize $\E[-\log q_\phi(\theta_S \mid \mathcal{D}, S)]$ over uniformly drawn subsets of all sizes, the best Gaussian output has the mean and variance of the posterior $\int p(\theta_S \mid \zeta)\,d\pi_0(\zeta \mid \mathcal{D})$ under its training prior $\pi_0$ (Proposition~\ref{prop:cnp}). It thus reports the variance of Equation~\eqref{eq:postvar}, cross terms included, without sampling. Because $S$ is an input, the model need not be retrained for new subsets.

\section{Experiments}
\label{sec:experiments}

\subsection{Setup}
\label{sec:expsetup}

\begin{table}[t]
\centering
\footnotesize
\setlength{\tabcolsep}{5pt}
\begin{tabular}{lllrr}
\toprule
\textbf{dataset} & \textbf{covariates} & \textbf{source} & \textbf{$n$} & \textbf{$d$} \\
\midrule
Synthetic family & simulated $\mathcal{N}(0, I)$ & Appendix~\ref{app:dgps} & 500 & 3 \\
Hill-A & simulated $\mathcal{N}(0, I)$ & \citet{hill2011}, setting A & 1000 & 25 \\
Friedman & simulated $\mathrm{U}(0,1)$ & \citet{friedman1991}, confounded & 1000 & 10 \\
IHDP & real, IHDP study & \citet{hill2011}; splits of \citet{shalit2017} & 747 & 25 \\
Hill-B & simulated $\mathcal{N}(0, I)$ & \citet{hill2011}, setting B & 1000 & 25 \\
Kang--Schafer & simulated, nonlinear in 4 latents & \citet{kang2007} & 1000 & 4 \\
ACIC 2016 & real, competition covariate table & \citet{dorie2019} & 4802 & 58 \\
\bottomrule
\end{tabular}
\caption{Datasets used in our experiments. Our models are trained and evaluated on the synthetic family, our heads on CausalPFN are evaluated on the six benchmarks, and CausalPFN is evaluated on all seven. Hill-A, Hill-B, Friedman, and Kang--Schafer are our reimplementations (Appendix~\ref{app:dgps}). $d$ is the number of covariates supplied to the estimator.}
\label{tab:benchmarks}
\end{table}

\paragraph{Datasets.} We train and evaluate our own models on a synthetic family of SCMs (Appendix~\ref{app:dgps}). Each task draws a new SCM with three covariates and $n=500$ units, and $\theta_S$ has a closed form for every subset. A sharpness parameter $s$ controls how far propensity scores are pushed toward 0 and 1; we use $s=1$ and $s=4$ for our models and $s \in \{1,2,4,6\}$ for CausalPFN. CausalPFN is evaluated on the same family and on six standard benchmarks (Table~\ref{tab:benchmarks}), and our heads on CausalPFN on the six benchmarks. Because all of them have simulated response surfaces, we score intervals against each unit's known conditional effect (Appendix~\ref{app:dgps}).

\paragraph{Models.} CausalPFN is used as released, without retraining, with intervals formed by its own aggregate estimator. On the benchmarks we enable its built-in per-dataset temperature calibration; on the synthetic family we use its default, which leaves calibration off. Our CNP and latent NP are trained at $s=1$, with three seeds each (Appendix~\ref{app:protocol}). To evaluate SCM-marginalized aggregation on the benchmarks, we also train both models as heads on the frozen CausalPFN, on a broader prior with $n$ between 100 and 5000 and up to 60 covariates, with three seeds each. Each head reads CausalPFN's last hidden states and its unit-level effect mean and standard deviation (Appendix~\ref{app:protocol}). No benchmark is used for training or checkpoint selection.

\paragraph{Experiments.} The \emph{benchmark} experiment applies CausalPFN to each benchmark, querying uniformly drawn subsets of increasing size on 40 datasets per benchmark. The \emph{fixed-weight} experiment holds a trained CNP or latent NP fixed and, for the same subset, forms one interval by unit-level aggregation, as in Equation~\eqref{eq:intro-unit}, and one by SCM-marginalized aggregation, as in Equation~\eqref{eq:intro-scm}. Because weights, context, subsets, and targets are shared, the comparison isolates the aggregation step (for the CNP, the point estimate also differs). It tests the mechanism and is not a comparison with CausalPFN. It covers six subset selection rules for the CNP and three for the latent NP, plus $s=4$ for the CNP with uniformly drawn subsets. We call each of the 60 combinations of trained model, rule, subgroup size, and $s$ a \emph{cell} (Appendix~\ref{app:protocol}). The \emph{head} experiment repeats the fixed-weight comparison with the heads on the six benchmarks, on the datasets and uniformly drawn subsets of the benchmark experiment. The \emph{variance-ratio} experiment compares coverage with Proposition~\ref{prop:design} for every model, rule, and subset size.

\paragraph{Metrics.} We report coverage of nominal $95\%$ intervals, estimated over tasks and subsets as in Appendix~\ref{app:protocol}, and $\hat k^*$, the subset size at which observed coverage first falls below nominal. Because $m_S$ and $\rho_S$ require the true effects, we use them only to explain the results.

\subsection{Results}
\label{sec:results}

\paragraph{Benchmarks.}\label{sec:benchmarks} Table~\ref{tab:legacy} in Appendix~\ref{app:additional} gives the full results behind Figure~\ref{fig:benchmarks}. Coverage starts at or above nominal on every benchmark in the figure and falls below it as subsets grow, fastest on ACIC 2016 and slowest on Hill-A, where it falls below nominal only beyond half the sample. Equation~\eqref{eq:intro-coverage} approximates the observed curves but overstates coverage on Hill-B and Kang--Schafer (Appendix~\ref{app:additional}). The error correlation alone does not predict where coverage fails. Hill-A has the largest correlation in the suite ($\rho_S = 0.077$) but is the last to fall below nominal, while Hill-B, with a correlation twenty times smaller, loses coverage at $5.8\%$ of its sample. Equation~\eqref{eq:kstar} accounts for the ordering, since Hill-A's unit-level intervals are far too wide ($m_S = 35$) and Hill-B's only slightly ($m_S = 1.42$). Accurate unit-level estimates do not ensure aggregate coverage either. Although Kang--Schafer has the second-lowest normalized RMSE in the suite ($0.09$), it loses coverage at $2.1\%$ of its sample.

\begin{figure}[t]
\centering
\includegraphics[width=0.94\textwidth]{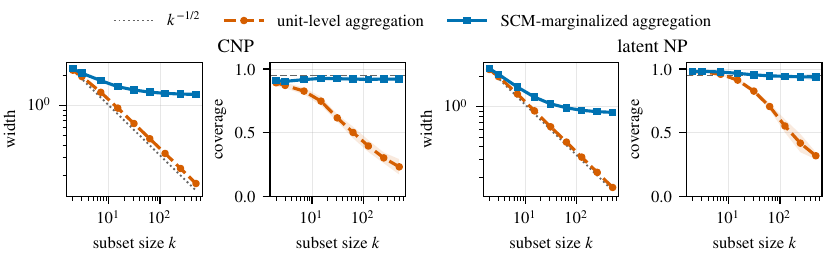}
\caption{Unit-level and SCM-marginalized aggregation from one CNP and one latent NP, with uniformly drawn subsets at $s=1$, showing interval width and coverage against subset size.}
\label{fig:twopath}
\end{figure}

\paragraph{Fixed weights.}\label{sec:twopath} In Figure~\ref{fig:twopath}, the unit-level interval shrinks at close to the $k^{-1/2}$ rate that independence implies, and at the largest subsets it is $7.7$ times narrower than the SCM-marginalized interval for the CNP and $5.5$ times for the latent NP. Unit-level coverage falls with $k$, while SCM-marginalized coverage stays between $0.90$ and $0.98$. Across all $60$ cells, SCM-marginalized coverage at the largest subset has a median of $0.89$ (range $0.70$ to $0.99$), and its lowest value exceeds the highest unit-level value ($0.47$). Coverage is lowest for the CNP at $s=4$, outside its training distribution, and for subsets selected on the true effect, which Corollary~\ref{cor:exact} excludes because $S$ depends on the SCM. The latent NP, whose decoder treats units as independent given $z$, has a median of $0.83$. In a separate comparison without fixed weights, with three models trained at $n=2000$, unit-level coverage is $0.17$ at $k=n$, and the two SCM-marginalized models reach $0.97$ and $0.98$ (Table~\ref{tab:n2000}).

\begin{figure}[t]
\centering
\IfFileExists{figures/cfm_heads_vs_causalpfn.pdf}{\includegraphics[width=\textwidth]{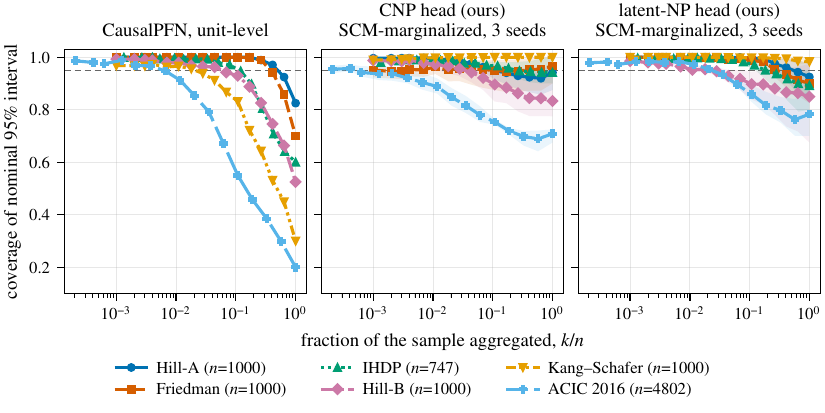}}{\fbox{\parbox{0.9\textwidth}{\centering\ttfamily MISSING FIGURE: figures/cfm\_heads\_vs\_causalpfn.pdf}}}
\caption{Coverage of nominal $95\%$ subset-CATE intervals on the six benchmarks, with subsets drawn uniformly at random. Left: CausalPFN with unit-level aggregation, as in Figure~\ref{fig:benchmarks}. Middle and right: SCM-marginalized aggregation from the CNP and latent NP heads on the frozen CausalPFN, with lines showing the mean over three seeds and bands their range. The heads condition on the whole sample, while the CausalPFN curves come from its public interface, which on ACIC 2016 conditions groups of units on their nearest neighbours (Appendix~\ref{app:additional}). The figure compares aggregation rules on the same benchmarks, not the accuracy of the models.}
\label{fig:heads}
\end{figure}

\paragraph{Heads on CausalPFN.}\label{sec:heads} Figure~\ref{fig:heads} and Table~\ref{tab:heads-coverage} in Appendix~\ref{app:additional} repeat the fixed-weight comparison with the added heads on the six benchmarks. At $k=n$, unit-level aggregation from the same weights covers between $0.10$ and $0.72$, and SCM-marginalized aggregation between $0.71$ and $1.00$. Over the six benchmarks, the full-sample interval contains the true value in $90\%$ of the $240$ datasets for the CNP head and $89\%$ for the latent NP head, against $25\%$ and $52\%$ under unit-level aggregation and $53\%$ for CausalPFN. It is furthest below nominal on ACIC 2016 ($0.71$ and $0.78$) and Hill-B ($0.83$ and $0.85$). The heads do not improve unit-level accuracy. Their normalized RMSE is lower than CausalPFN's on Hill-A, Friedman, and Kang--Schafer and higher on Hill-B ($0.52$ against $0.28$), IHDP ($0.20$ and $0.21$ against $0.11$), and ACIC 2016 ($0.30$ against $0.27$), and at $k=n$ their intervals are wider than CausalPFN's on every benchmark (Table~\ref{tab:heads-moments}).

\begin{figure}[t]
\centering
\includegraphics[width=0.6\textwidth]{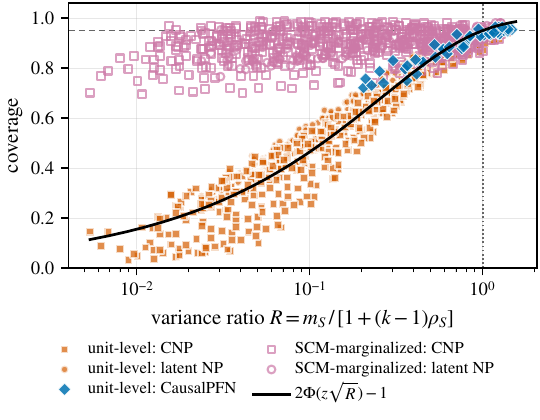}
\caption{Measured coverage against the variance ratio $R$ of Proposition~\ref{prop:design}, pooled across models, values of $s$, subset selection rules, and subset sizes. The curve is Equation~\eqref{eq:intro-coverage}, with $m_S$ and $\rho_S$ measured from unit-level errors. SCM-marginalized points are plotted at the $R$ of unit-level aggregation.}
\label{fig:designeffect}
\end{figure}

\paragraph{Variance ratio.}\label{sec:designeffect} Figure~\ref{fig:designeffect} plots observed coverage against $R$. We compare it with the prediction for $R<1$, since for $R \ge 1$ the predicted coverage is at least nominal. Unit-level coverage closely follows Equation~\eqref{eq:intro-coverage}, with a median absolute error of $0.043$ for our models and $0.021$ for CausalPFN on the synthetic family. Because Equations~\eqref{eq:design} and~\eqref{eq:bias} are exact given $m_S$ and $\rho_S$, the figure mainly tests the Gaussian and fixed-variance assumptions. SCM-marginalized points lie well above the curve (median absolute error $0.315$), since their intervals include the cross terms. On the benchmarks, the prediction for unit-level aggregation from the heads has a median absolute error of $0.029$ for the CNP head and $0.042$ for the latent NP head over the points with $R<1$, with its largest error on IHDP for both heads ($0.23$ and $0.15$; Appendix~\ref{app:additional}).

\section{Discussion and limitations}
\label{sec:limitations}

Our CNP and latent NP are small models trained on a synthetic family; we do not claim they match CausalPFN's accuracy or generality. As heads on a frozen CausalPFN, they recover most of the coverage lost on the benchmarks but not all of it, and their unit-level estimates are less accurate than CausalPFN's on three of the six, most of all on Hill-B and IHDP. Because our family has few parameters, a small latent can capture most of the posterior dependence, but richer priors may require a larger one. On benchmarks, part or all of $\rho_S$ can come from the shared bias term $b_ib_j$ of Proposition~\ref{prop:source}, which no aggregation rule removes.

For existing models, rescaling unit-level intervals to be calibrated did not help in our experiments, and it made coverage worse when they were too wide (Appendix~\ref{app:correction}). Multiplying the aggregate variance by the design effect brings coverage close to nominal, but it requires $\rho_S$ and $m_S$. They could be estimated on simulated tasks, but they vary across our benchmarks by more than an order of magnitude (Table~\ref{tab:legacy}) and depend on the model as well as the subgroup.

\section{Conclusion}
\label{sec:conclusion}

Causal foundation models commonly form intervals for effects averaged over the full sample, the treated units, or a subgroup by combining unit-level intervals as if the errors were independent. Because posterior uncertainty over the SCM makes these errors correlated, the combined variance is too small by the design effect, and coverage falls as the group grows. We observe this on all six standard benchmarks, trace it to the aggregation step with the weights held fixed, and predict its size from $m_S$ and $\rho_S$. Integrating over the SCM jointly for all units recovers the missing covariance. Both models we test that do so give much higher coverage than unit-level aggregation, on the synthetic family and as heads on a frozen CausalPFN on the six benchmarks, though not always at the nominal level.

\subsubsection*{Use of AI tools}
We used large language models to aid and polish writing, for retrieval and discovery of related work, and to draft sections of this paper. All drafted text was reviewed and revised by the authors, all cited work was verified against primary sources, and all theoretical claims and experimental results are the authors' own.

\subsubsection*{Reproducibility statement}
Appendix~\ref{app:dgps} specifies the synthetic family in full, including the coefficient priors and the propensity mechanism, and gives the covariate source, sample size, and generating design for each of the six benchmarks. Appendix~\ref{app:protocol} gives the architectures, the training setup and checkpoint selection for each model, and the task and subset counts behind every coverage estimate we report. Appendix~\ref{app:causalpfn-equiv} identifies the CausalPFN code paths our characterization relies on, at a pinned commit. CausalPFN is used as released, without retraining, with the calibration settings of Section~\ref{sec:expsetup}. Code to reproduce every figure and table, with the recorded configuration and training command for each of our trained models, is available at \url{https://github.com/gabeschulman-nyu/subset-cate-coverage}.

\bibliography{refs}
\bibliographystyle{iclr2027_conference}

\appendix
\let\appendixsection\section
\renewcommand{\section}{\FloatBarrier\appendixsection}
\setcounter{topnumber}{3}
   \renewcommand{\floatpagefraction}{0.9}

\section{Data-generating processes}
\label{app:dgps}

We use a synthetic family so that the true subset-CATE can be computed exactly for any subset and the conditions under which coverage fails can be varied.

\subsection{Synthetic family}

Each task draws a fresh set of coefficients, then $n$ units from the resulting SCM. Covariates are $X_\ell \sim \mathcal{N}(0,1)$ independently for $\ell = 1,\dots,d$ with $d=3$, and each unit is summarized by the feature vector
\begin{equation}
h(X) = [\,X_1,\dots,X_d,\;X_1^2,\dots,X_d^2\,] \in \mathbb{R}^{6},
\end{equation}
which includes squared terms so that both treatment assignment and outcome may depend nonlinearly on the covariates.

Treatment is assigned according to a propensity score $p(X) = \Pr(T=1 \mid X)$, and the outcome is generated as
\begin{align}
T \mid X &\sim \mathrm{Bernoulli}\big(p(X)\big), &
p(X) &= 0.05 + 0.90\,\mathrm{expit}\big(s\,h(X)^\top w_p\big), \\
Y &= h(X)^\top w_b + T\,\tau(X) + u, &
u &\sim \mathcal{N}(0, 0.25^2).
\end{align}
The sharpness $s$ controls how far propensity scores are pushed toward 0 and 1. For any $s$, $p(X)$ stays in $[0.05, 0.95]$, so positivity holds. Our models are trained only at $s=1$, and the CNP is also evaluated at $s=4$, outside its training distribution. We evaluate CausalPFN at $s \in \{1,2,4,6\}$, and the settings above $s=1$ account for $25$ of its $29$ points with $R<1$ in Figure~\ref{fig:designeffect}.

The CATE is
\begin{equation}
\tau(X) = 0.25 + h(X)^\top w_e.
\end{equation}

The covariates enter both the propensity score and the outcome, so they are confounders. Since they are fully observed, $\tau$ is identified.

Since $\tau(X)$ is available in closed form, $\theta_S$ is computed exactly for every task and subset, and targets add no Monte Carlo error to the coverage estimates.

The per-task coefficients are drawn as
\begin{equation}
w_b \sim \mathcal{N}(0, 0.5^2 I_6), \qquad
w_e \sim \mathcal{N}(0, 0.35^2 I_6), \qquad
w_p \sim \mathcal{N}(0, 0.35^2 I_6),
\end{equation}
independently across tasks.

Given the coefficients, units are drawn independently, so they are i.i.d.\ given the SCM $\zeta$ and exchangeable but not independent marginally, as in Section~\ref{sec:definetti}.

\subsection{Benchmarks}

We evaluate CausalPFN on the six standard benchmarks of Table~\ref{tab:benchmarks} and on a heavy-tailed variant of Hill-B. In each, the response surfaces are simulated, so each unit's conditional effect $\mu_1(X_i)-\mu_0(X_i)$ is known. We score intervals against it, not against the difference of the two noisy potential outcomes, and there is no unobserved confounding by construction. IHDP and ACIC 2016 pair the simulated surfaces with covariates from real studies; IHDP uses the \texttt{ihdp\_npci\_1-100} splits and ACIC 2016 the official \texttt{aciccomp2016} generator. The other designs simulate the covariates as well, with $n=1000$ and independent $\mathcal{N}(0,1)$ noise on each potential outcome.
\begin{itemize}[leftmargin=*, itemsep=1pt, topsep=2pt]
\item \textbf{Hill-A.} Twenty-five $\mathcal{N}(0,1)$ covariates, a logistic propensity in the first five, $\mu_0 = X\beta$, and $\mu_1 = X\beta + 4$ with $\beta \sim \mathcal{N}(0, 0.5^2 I)$, so the effect is constant at 4, as in surface A of \citet{hill2011}.
\item \textbf{Hill-B.} The same covariates and propensity, $\mu_0 = \exp\big((X+0.5)\beta_B\big)$ with the entries of $\beta_B$ drawn from $\{0, 0.1, 0.2, 0.3, 0.4\}$ with probabilities $(0.6, 0.1, 0.1, 0.1, 0.1)$, and $\mu_1 = X\gamma + c$ with $\gamma \sim \mathcal{N}(0, 0.3^2 I)$ and $c$ set so that the average effect on the treated is 4. Unlike \citet{hill2011}, $\gamma$ is drawn separately from $\beta_B$. The heavy-tailed variant replaces the Gaussian noise with $t_3$ noise.
\item \textbf{Friedman.} Ten $\mathrm{U}(0,1)$ covariates, a logistic propensity in $2(X_1-0.5)+1.5(X_3-0.5)$, $\mu_0 = 10\sin(\pi X_1X_2)+20(X_3-0.5)^2+10X_4+5X_5$ as in \citet{friedman1991}, and effect $2+X_1-X_2$.
\item \textbf{Kang--Schafer.} Four latent variables $Z \sim \mathcal{N}(0, I)$, of which the estimator sees only the transformed covariates of \citet{kang2007}, a logistic propensity in $Z$, and $\mu_0 = \mu_1 = 210 + 27.4Z_1 + 13.7(Z_2+Z_3+Z_4)$, so the effect is zero for every unit.
\end{itemize} We use splits 1--40 of 100 for IHDP, settings 1--40 of 77 for ACIC 2016 (one replication each), and seeds 1--40 for the simulated designs, for 40 datasets per benchmark. In Hill-A, Friedman, and Kang--Schafer the effect is the same function of the covariates in every dataset, so under the benchmark's own distribution the posterior covariance in Proposition~\ref{prop:source} is zero and all of $\rho_S$ comes from the shared bias term $b_ib_j$.

\section{Training and evaluation protocol}
\label{app:protocol}

\paragraph{Architectures.}
Both models encode each unit's covariates, treatment, and outcome, and mean-pool the encodings into a summary of the sample. The CNP also pools the encodings of the queried subset and decodes a Gaussian for $\theta_S$ from the two summaries in a single pass. The latent NP adds a 16-dimensional Gaussian latent $z$ whose mean and scale are computed from the sample summary, decodes each unit's effect given $z$ as a 64-bin histogram on $[-4, 4]$ in standardized units, and forms $\theta_S$ within each draw of $z$. Every network is a two-layer ReLU MLP of width 128, and neither model uses attention.

\paragraph{Training.}
Both models are trained on the synthetic family of Appendix~\ref{app:dgps} at $s=1$ and $n=500$. The CNP is trained on subsets drawn uniformly at random, with size sampled log-uniformly between $1$ and $n$, by Gaussian negative log-likelihood. The latent NP's loss uses 16 draws of $z$ and has two terms. The first is the cross-entropy of each unit's histogram against a Gaussian-smoothed target at its true effect, averaged over draws. The second, with weight 2, is the negative log of the average over draws of a Gaussian likelihood for $\theta_S$, with $S$ drawn as for the CNP and, within each draw, the Gaussian matching the mean and variance of the average of the decoded unit effects. There is no KL term, and $z$ is drawn from the same distribution at inference, so this is a Monte Carlo likelihood objective similar to that of \citet{foong2020convnp}. The latent NP standardizes each task's outcome and covariates, and the CNP works on the raw scale, in both training and evaluation. We use Adam with learning rate $3 \times 10^{-4}$ and cosine decay to $3 \times 10^{-6}$, batches of 64 tasks, and $30{,}000$ steps. We train three seeds of each model. One seed takes about one GPU-hour for the CNP and under two for the latent NP.

\paragraph{Subset selection rules.}
The fixed-weight experiment uses six rules, each applied to a sample of $n=500$: uniformly drawn subsets of the whole sample; the top 25\% of units by $X_1$; the bottom fraction by true propensity score; a random anchor unit and its nearest neighbours in covariate space; the top fraction by the model's own predicted effect; and the top fraction by the true effect, an oracle rule. For the last four rules, the fraction is $0.10$, $0.25$, or $0.50$. The CNP uses all six rules at $s=1$ and uniformly drawn subsets at $s=4$, for 15 settings; the latent NP uses the first three rules at $s=1$, for 5 settings. With three seeds each, this gives $45$ CNP cells and $15$ latent NP cells. Uniformly drawn subsets vary $k$ at a fixed error correlation, and the other rules select subgroups, as an analyst would. For the latent NP, the two intervals share the unit-level marginals and differ only in the aggregation step; for the CNP, they also differ in the point estimate, since the SCM-marginalized interval is decoded directly.

\paragraph{Checkpoint selection.}
We evaluate the checkpoint with the best score over the full run on 256 held-out validation tasks, using negative log-likelihood for the CNP and per-unit RMSE for the latent NP.

\paragraph{Heads on CausalPFN.}
For the benchmarks, we train both models as heads on the released CausalPFN (version 0.1.4), whose 18.8M parameters are frozen. Each dataset takes one forward pass with all $n$ units as context, and each unit is queried under $t=0$ and $t=1$. A unit's input to the head is CausalPFN's last hidden state under both treatments, which its own output layer reads, together with its unit-level effect mean and standard deviation in units of $\mathrm{sd}(y)$. The CNP head (461k parameters) applies a small network to each unit, averages the outputs over the sample and over $S$, and decodes a Gaussian for $\theta_S/\mathrm{sd}(y)$ from the two averages, $\log k$, and $k/n$. The latent NP head (621k parameters) draws a 16-dimensional $z$ from a Gaussian computed from the sample average and decodes each unit's effect given $z$ as a 64-bin histogram on $[-4, 4]$ in units of $\mathrm{sd}(y)$; effects outside this range, $0.9\%$ of units on Hill-B and $0.3\%$ on ACIC 2016, fall in the end bins. Both are trained with the losses and optimizer above, with gradients clipped at norm $1.0$. Training tasks are generated online and never reused, about $5$M per run at about 20 tasks per step over $250{,}000$ steps, with $n$ log-uniform between 100 and 5000 and between 1 and 60 covariates. Half of the tasks draw covariates from 330 real OpenML tables, excluding the source tables of the benchmarks, and half simulate them. Outcome surfaces are MLPs, trees, or additive functions with heterogeneous effects, propensities lie in $[0.02, 0.98]$, and noise is Gaussian, $t$, or skewed. Each task contributes 32 subsets of log-uniform size, instead of one. Each run keeps the checkpoint with the lowest validation loss on 256 held-out prior tasks, checked every 5000 steps, using the negative log-likelihood of $\theta_S$ for the CNP head and the training objective for the latent NP head, where the latent NP above is selected on per-unit RMSE. We train three seeds of each head, and the same seed gives both heads the same task stream. For the CNP head, unit-level aggregation queries the head with each unit alone as the subset and combines the resulting Gaussians as in Equation~\eqref{eq:sigma}. For the latent NP head, each unit receives its own independent draw of $z$.

\paragraph{Estimating coverage.}
We estimate $\mathcal{C}(k)$ of Equation~\eqref{eq:coverage-def} over $M$
tasks with $L$ subsets sampled per task,
\begin{equation}
\widehat{\mathcal{C}}(k) \;=\; \frac{1}{ML}\sum_{j=1}^{M}\sum_{l=1}^{L}
  \mathbf{1}\big\{\theta_{S_{jl}}\in[\hat{l}_{S_{jl}},\hat{u}_{S_{jl}}]\big\}.
\label{eq:coverage-est}
\end{equation}
Standard errors on $\widehat{\mathcal{C}}(k)$ are clustered by task, since subsets from one task share its error component. The fixed-weight experiment uses 64 tasks with 24 subsets each, the $n = 2000$ comparison 96 tasks with 24 subsets each, and the benchmarks 40 datasets with 48 subsets each.

\paragraph{Measuring $\rho_S$ and $m_S$.}
We compute $\sigma_S^2$ and $c_S$ as the empirical second moments of Equation~\eqref{eq:moments}, from unit-level errors against the analytic targets of Appendix~\ref{app:dgps}, and $m_S$ as the model's average reported unit-level variance divided by $\sigma_S^2$. Both $\rho_S$ and $m_S$ require the true effects, so neither can be computed in practice, and we use them only as diagnostics. On the benchmarks, we pool errors in outcome units over the 40 datasets, compute the moments at each subset size from one subset per dataset, and report the median over sizes. Where error scales differ widely across datasets, the estimate depends on this choice. For the CNP head on Hill-B, for example, $m_S$ is $0.84$ with this estimator and $0.44$ when all 48 subsets per dataset and size are used, and Equation~\eqref{eq:intro-coverage} is correspondingly less reliable there (Appendix~\ref{app:proofs}).

\section{Correcting for the design effect}
\label{app:correction}

Correcting $m_S$ and then $\rho_S$ separates the contribution of unit-level calibration from that of the aggregation step.

Table~\ref{tab:threearm} compares three intervals built from the same unit-level outputs. The \emph{reported} interval is the model's own. The \emph{rescaled} interval divides the unit-level variances by $m_S$, the standard correction for miscalibrated unit-level intervals, and then aggregates at the unit level as before. The \emph{$+$ design effect} interval additionally multiplies the aggregate variance by $1+(k-1)\rho_S$. The CNP row averages its three seeds, each corrected with its own $\rho_S$ and $m_S$.

\begin{table}[tp]
\centering
\footnotesize
\setlength{\tabcolsep}{5pt}
\begin{tabular}{llrrlrr}
\toprule
\textbf{model} & \textbf{data} & \textbf{$m_S$} & \textbf{$k_{\max}$} & \textbf{interval} & \textbf{cov, smallest $k$} & \textbf{cov, $k_{\max}$} \\
\midrule
CausalPFN & $s=1$, all units & 21.90 & 500 & reported & 1.00 & 0.94 \\
 &  &  &  & rescaled & 0.91 & 0.41 \\
 &  &  &  & $+$ design effect & 0.91 & 0.93 \\
\addlinespace[2pt]
CausalPFN & $s=4$, bottom 10\% & 3.63 & 44 & reported & 0.99 & 0.79 \\
 &  &  &  & rescaled & 0.96 & 0.58 \\
 &  &  &  & $+$ design effect & 0.96 & 0.93 \\
\addlinespace[2pt]
CausalPFN & IHDP & 5.43 & 747 & reported & 1.00 & 0.60 \\
 &  &  &  & rescaled & 0.98 & 0.25 \\
 &  &  &  & $+$ design effect & 0.98 & 0.93 \\
\addlinespace[2pt]
CNP (ours) & $s=1$ & 0.93 & 500 & unit-level & 0.89 & 0.23 \\
 &  &  &  & rescaled & 0.91 & 0.24 \\
 &  &  &  & $+$ design effect & 0.92 & 0.93 \\
\bottomrule
\end{tabular}
\caption{Coverage of the three intervals of Appendix~\ref{app:correction} at the smallest and largest subset in each setting. $k_{\max}$ is the largest evaluated subset size within the pool. The \emph{$+$ design effect} interval uses $\rho_S$ and $m_S$ measured against the true effects and is not available in practice.}
\label{tab:threearm}
\end{table}

In the three CausalPFN rows $m_S > 1$, and rescaling lowers coverage at both the smallest subset and $k_{\max}$. Under unit-level aggregation, only over-wide unit-level intervals offset the design effect, and rescaling removes them without correcting the aggregation. The CNP's unit-level intervals are close to calibrated ($m_S = 0.93$), so rescaling changes little and coverage at $k_{\max}$ stays near $0.24$. CausalPFN's own temperature calibration acts like rescaling. Its full-sample interval on ACIC 2016 covers the true value in 18 of 40 datasets without it and in 8 with it.

The design-effect interval removes the dependence on $k$, as Equation~\eqref{eq:bias} implies, but it requires $\rho_S$ and $m_S$ measured against the true effects and so cannot be used in practice.

\section{Subgroup selection}
\label{app:subgroups}

In Table~\ref{tab:subset-selection}, subsets are selected on the propensity score or on covariates, as an analyst querying a subgroup would select them, instead of uniformly at random.

\begin{table}[tp]
\centering
\footnotesize
\setlength{\tabcolsep}{4pt}
\begin{tabular}{lrrrrr}
\toprule
\textbf{subset} & \textbf{$\rho_S$} & \textbf{$m_S$} & \textbf{$k_{\max}$} & \textbf{$1+(k_{\max}\!-\!1)\rho_S$} & \textbf{$\hat k^*$} \\
\midrule
\multicolumn{6}{l}{\itshape $s=1$, propensity stratum} \\
all units (baseline) & 0.034 & 21.90 & 500 & 17.9 & 398 \\
bottom 50\% & 0.069 & 23.23 & 250 & 18.3 & 216 \\
bottom 25\% & 0.120 & 14.10 & 125 & 15.9 & 83 \\
bottom 10\% & 0.224 & 8.07 & 44 & 10.6 & 34 \\
\addlinespace[3pt]
\multicolumn{6}{l}{\itshape $s=1$, nested covariate cell} \\
$X_1>0$ & 0.043 & 22.89 & 177 & 8.7 & --- \\
$\wedge\, X_2>0$ & 0.076 & 23.83 & 89 & 7.7 & --- \\
$\wedge\, X_3>0$ & 0.122 & 26.46 & 31 & 4.7 & --- \\
\addlinespace[3pt]
\multicolumn{6}{l}{\itshape $s=1$, covariate neighbourhood} \\
nearest 50\% & 0.075 & 45.46 & 250 & 19.6 & --- \\
nearest 25\% & 0.143 & 46.08 & 125 & 18.7 & --- \\
nearest 10\% & 0.231 & 36.09 & 44 & 10.9 & --- \\
\addlinespace[3pt]
\multicolumn{6}{l}{\itshape $s=4$, propensity stratum} \\
all units (baseline) & 0.059 & 9.20 & 500 & 30.4 & 93 \\
bottom 50\% & 0.076 & 9.72 & 250 & 20.0 & 89 \\
bottom 25\% & 0.142 & 6.12 & 125 & 18.6 & 31 \\
bottom 10\% & 0.257 & 3.63 & 44 & 12.1 & 8 \\
\addlinespace[3pt]
\multicolumn{6}{l}{\itshape $s=4$, nested covariate cell} \\
$X_1>0$ & 0.069 & 9.41 & 177 & 13.1 & 86 \\
$\wedge\, X_2>0$ & 0.095 & 9.58 & 89 & 9.4 & 58 \\
$\wedge\, X_3>0$ & 0.176 & 8.44 & 31 & 6.3 & --- \\
\bottomrule
\end{tabular}
\caption{Subset selection rules on CausalPFN at $n=500$, with intervals formed by unit-level aggregation, over $64$ tasks with $48$ subsets per task. $k_{\max}$ is the largest evaluated subset size within the pool and $1+(k_{\max}\!-\!1)\rho_S$ the design effect at that size; ``---'' means coverage never falls below nominal within the pool. $\hat k^*$ is defined in Section~\ref{sec:expsetup}.}
\label{tab:subset-selection}
\end{table}

Narrowing a pool raises $\rho_S$ but shrinks $k_{\max}$, so the largest reachable design effect changes much less than either and falls for the narrowest pools.

Narrow propensity strata also lower $m_S$, whereas nested covariate cells leave it roughly unchanged and covariate neighbourhoods raise it. A rule that raises $\rho_S$ without lowering $m_S$ need not reduce coverage below nominal.

At $s=1$, nested covariate cells reach $\rho_S$ comparable to the bottom 25\% propensity stratum and still never fall below nominal, because $m_S$ stays high. At $s=4$, where $m_S$ starts lower, the first two covariate cells do.

\appendixsection{How existing models construct subset-CATE intervals}
\label{app:causalpfn-equiv}

\paragraph{CausalPFN.} Each unit's predictive distribution is a histogram over 1024 bins \citep{causalpfn2025}. The subset-CATE point estimate is the average of the relevant unit-level point estimates. For the credible interval, its code draws one sample from each unit's predictive distribution, averages those samples across the units in $S$, repeats this many times, and takes percentiles of the resulting distribution of averages. For every subset, including the full sample, we apply this procedure to per-unit draws from CausalPFN's sampler.

Let $\hat\tau_i^{(b)}$ denote the $b$-th draw from unit $i$'s predictive distribution, and let $\hat\theta^{(b)} = k^{-1}\sum_{i \in S}\hat\tau_i^{(b)}$ be the corresponding draw of the subset-CATE. Given $\mathcal{D}$, draws for different units are generated independently and share no randomness, so the covariance terms vanish and the variance of each aggregate draw is exactly
\begin{equation}
\Var\big(\hat\theta^{(b)}\big)
= \frac{1}{k^2}\sum_{i \in S}\Var\big(\hat\tau_i^{(b)}\big)
+ \frac{1}{k^2}\sum_{\substack{i,j \in S \\ i \neq j}}\Cov\big(\hat\tau_i^{(b)}, \hat\tau_j^{(b)}\big)
= \frac{1}{k^2}\sum_{i \in S}\hat\sigma_i^2 \;=\; \hat\sigma_S^2 ,
\end{equation}
where $\hat\sigma_i^2$ is the variance of unit $i$'s predictive distribution. This is Equation~\eqref{eq:sigma}, and the percentiles the code returns are Monte Carlo estimates of the quantiles of this distribution.

The interval the code returns is a percentile interval of the sampled aggregate distribution instead of the symmetric interval $\hat\theta_S \pm z_{1-\alpha/2}\hat\sigma_S$. The two coincide when the sampled distribution is symmetric, and the average of $k$ independent draws approaches a Gaussian as $k$ grows, so the discrepancy is largest at small $k$ and negligible at the subset sizes where coverage fails. Whether $\hat\sigma_i^2$ matches unit $i$'s actual squared error is a separate question, which $m_S$ measures.

\paragraph{Do-PFN and CausalFM.} Neither provides an aggregate estimator. The only uncertainty quantification in Do-PFN's repository, which is not part of its estimation interface, builds an interval from the standard error across units and so also assumes independence. CausalFM returns a Gaussian mixture for each unit, and the interval construction in its documentation samples each unit's mixture independently and forms a separate interval for each unit. The only aggregate quantity in its documentation is an average of unit-level point estimates, with no associated credible interval.

\section{Additional results}
\label{app:additional}

\begin{table}[tp]
\centering
\footnotesize
\setlength{\tabcolsep}{4pt}
\begin{tabular}{lrrrrrrr}
\toprule
\textbf{benchmark} & \textbf{$n$} & \textbf{$\rho_S$} & \textbf{$m_S$} & \textbf{nRMSE} & \textbf{cov at $k$=$n$} & \textbf{se} & \textbf{$\hat k^*$ ($\hat k^*/n$)} \\
\midrule
Hill-A & 1000 & 0.0771 & 35.13 & 0.07 & 0.82 & 0.06 & 502 (50\%) \\
Friedman & 1000 & 0.0334 & 11.75 & 0.10 & 0.70 & 0.07 & 383 (38\%) \\
IHDP & 747 & 0.0420 & 5.43 & 0.11 & 0.60 & 0.08 & 85 (11\%) \\
Hill-B heavy-tailed & 1000 & 0.0077 & 2.03 & 0.30 & 0.53 & 0.08 & 62 (6.2\%) \\
Hill-B & 1000 & 0.0038 & 1.42 & 0.28 & 0.53 & 0.08 & 58 (5.8\%) \\
Kang--Schafer & 1000 & 0.0185 & 1.43 & 0.09 & 0.30 & 0.07 & 21 (2.1\%) \\
ACIC 2016 & 4802 & 0.0202 & 1.63 & 0.27 & 0.20 & 0.06 & 30 (0.6\%) \\
\bottomrule
\end{tabular}
\caption{The benchmarks of Figure~\ref{fig:benchmarks} and the heavy-tailed variant of Hill-B, which the figure omits: CausalPFN with its temperature calibration, unit-level aggregation, subsets drawn uniformly at random, 40 datasets per benchmark. $\hat k^*$ is the subset size at which observed coverage first falls below nominal, interpolated in $\log k$, with $\hat k^*/n$ as a percentage of the sample. nRMSE is the root-mean-square error of the per-unit effect estimates divided by the standard deviation of the observed outcome, within each dataset and averaged over datasets. Rows are ordered by $\hat k^*/n$.}
\label{tab:legacy}
\end{table}

\begin{table}[tp]
\centering
\footnotesize
\setlength{\tabcolsep}{4pt}
\begin{tabular}{llrrrr}
\toprule
\textbf{model} & \textbf{aggregation} & \textbf{$k$=5} & \textbf{$k$=50} & \textbf{$k$=500} & \textbf{$k$=2000} \\
\midrule
CNP & unit-level & 0.93 & 0.61 & 0.29 & 0.17 \\
latent NP & SCM-marginalized & 1.00 & 0.98 & 0.97 & 0.97 \\
CNP & SCM-marginalized & 0.95 & 0.98 & 0.98 & 0.98 \\
\bottomrule
\end{tabular}
\caption{Coverage of nominal $95\%$ subset-CATE intervals at $n=2000$ over 96 tasks with 24 subsets each. The two CNP rows are separately trained models. Each model was trained once, with optimization settings that differ from Appendix~\ref{app:protocol} and from each other, so the comparison is not controlled for optimization.}
\label{tab:n2000}
\end{table}

\begin{table}[tp]
\centering
\footnotesize
\setlength{\tabcolsep}{4pt}
\begin{tabular}{lrrrrr}
\toprule
\textbf{$n$} & \textbf{$\rho_S$} & \textbf{$m_S$} & \textbf{$\hat k^*$} & \textbf{$1+(n-1)\rho_S$} & \textbf{cov at $k$=$n$} \\
\midrule
125 & 0.0385 & 7.03 & --- & 6 & 0.95 $\pm$ 0.02 \\
250 & 0.0422 & 14.25 & --- & 12 & 0.96 $\pm$ 0.02 \\
500 & 0.0339 & 21.90 & 398 & 18 & 0.94 $\pm$ 0.03 \\
1000 & 0.0230 & 32.67 & 980 & 24 & 0.95 $\pm$ 0.02 \\
2000 & 0.0230 & 40.08 & 1376 & 47 & 0.93 $\pm$ 0.03 \\
\bottomrule
\end{tabular}
\caption{CausalPFN on the synthetic family of Appendix~\ref{app:dgps} at five sample sizes, with subsets drawn uniformly at random, 64 tasks and 48 subsets per task; coverage at $k=n$ is given $\pm$ one standard error clustered by task. The same CausalPFN model is used throughout, so no training choice varies down the table. $1+(n-1)\rho_S$ is the design effect at the full sample; ``---'' means coverage never falls below nominal. $\hat k^*$ is defined in Section~\ref{sec:expsetup}.}
\label{tab:nsweep}
\end{table}

\paragraph{Sample size.} Whether $\rho_S$ falls with $n$ depends on the prior. When a few shared parameters drive the effect, as in our synthetic family, errors on different units stay correlated as the posterior concentrates, so $\rho_S$ need not fall at all. If it fell as $n^{-1}$, the design effect at a fixed fraction of the sample would stay constant in $n$. Table~\ref{tab:nsweep} gives results for CausalPFN on the synthetic family of Appendix~\ref{app:dgps} at five sample sizes.

The decay is far slower than $n^{-1}$. A power law fitted to the five points has exponent $-0.24$, and over the sixteenfold range of $n$ in the table $\rho_S$ falls by a factor of $1.7$. The largest subset grows with $n$ while the correlation barely falls, so the design effect at the full sample, $1+(n-1)\rho_S$, grows from $6$ at $n=125$ to $47$ at $n=2000$.

Coverage at $k=n$ nonetheless stays close to nominal across the table, because $m_S$ grows as well. CausalPFN's per-unit errors shrink faster than the intervals it reports for them, so $m_S$ rises from $7$ to $40$, with a fitted exponent of $0.62$, and the numerator of Equation~\eqref{eq:bias} grows about as fast as the denominator. On the benchmarks of Section~\ref{sec:benchmarks}, $m_S$ lies between $1.42$ and $35.1$, and coverage falls below nominal on every one of them.

\paragraph{Accuracy of the predictions.} The measured $\hat k^*$ is within 15\% of Equation~\eqref{eq:kstar} on Hill-A, Kang--Schafer, and ACIC 2016, 19\% larger on Friedman, and 20\% smaller on IHDP. On Hill-B and its heavy-tailed variant, it is about half the prediction, and Equation~\eqref{eq:intro-coverage} overstates full-sample coverage there ($0.72$ and $0.66$, against $0.53$) and on Kang--Schafer ($0.41$, against $0.30$). One possible reason is that the predictions use moments pooled over the 40 datasets. Coverage is concave in $R$, so variation in $R$ across datasets can lower pooled coverage below the prediction (Appendix~\ref{app:proofs}). On the synthetic family (Table~\ref{tab:nsweep}), coverage stays within sampling error of nominal, so $\hat k^*$ is noisy there.

\paragraph{The heavy-tailed variant.} Hill-B heavy-tailed is omitted from Figure~\ref{fig:benchmarks} for legibility and appears in Table~\ref{tab:legacy}. Replacing Hill-B's Gaussian noise with $t_3$ noise changes little. Coverage falls below nominal at $6.2\%$ of the sample, against $5.8\%$ for Hill-B, and both reach $0.53$ at the full sample.

\begin{table}[tp]
\centering
\footnotesize
\setlength{\tabcolsep}{3pt}
\begin{tabular}{lrrrrr}
\toprule
 & \textbf{CausalPFN} & \multicolumn{2}{c}{\textbf{CNP head}} & \multicolumn{2}{c}{\textbf{latent NP head}} \\
\cmidrule(lr){3-4}\cmidrule(lr){5-6}
\textbf{benchmark} & \textbf{unit-level} & \textbf{unit-level} & \textbf{SCM-marg.} & \textbf{unit-level} & \textbf{SCM-marg.} \\
\midrule
Hill-A & 0.82 & 0.18 {\scriptsize[.12, .20]} & 0.94 {\scriptsize[.88, .97]} & 0.62 {\scriptsize[.57, .70]} & 0.92 {\scriptsize[.85, .97]} \\
Hill-B & 0.53 & 0.30 {\scriptsize[.23, .35]} & 0.83 {\scriptsize[.78, .95]} & 0.42 {\scriptsize[.25, .57]} & 0.85 {\scriptsize[.68, .97]} \\
Friedman & 0.70 & 0.23 {\scriptsize[.20, .25]} & 0.97 {\scriptsize[.90, 1.00]} & 0.68 {\scriptsize[.68, .70]} & 0.90 {\scriptsize[.82, .97]} \\
Kang--Schafer & 0.30 & 0.38 {\scriptsize[.35, .42]} & 1.00 {\scriptsize[1.00, 1.00]} & 0.72 {\scriptsize[.68, .78]} & 0.98 {\scriptsize[.97, 1.00]} \\
IHDP & 0.60 & 0.30 {\scriptsize[.20, .42]} & 0.94 {\scriptsize[.90, 1.00]} & 0.48 {\scriptsize[.47, .50]} & 0.89 {\scriptsize[.80, .95]} \\
ACIC 2016 & 0.20 & 0.10 {\scriptsize[.07, .15]} & 0.71 {\scriptsize[.68, .72]} & 0.18 {\scriptsize[.17, .20]} & 0.78 {\scriptsize[.70, .88]} \\
\addlinespace[2pt]
covered, of 240 & 126 & 59 {\scriptsize[54, 68]} & 216 {\scriptsize[211, 222]} & 125 {\scriptsize[117, 131]} & 213 {\scriptsize[207, 217]} \\
\bottomrule
\end{tabular}
\caption{Coverage of nominal $95\%$ intervals for the full-sample average ($k=n$) on the six benchmarks, 40 datasets each. CausalPFN's column is the full-sample column of Table~\ref{tab:legacy}. The heads' columns give the mean over three seeds, with the range of the per-seed values in brackets. Standard errors clustered by dataset are at most $0.06$ for the heads. The last row counts datasets whose full-sample interval contains the true value.}
\label{tab:heads-coverage}
\end{table}

\begin{table}[tp]
\centering
\footnotesize
\setlength{\tabcolsep}{4pt}
\begin{tabular}{llrrrrr}
\toprule
\textbf{benchmark} & \textbf{model} & \textbf{$m_S$} & \textbf{$\rho_S$} & \textbf{nRMSE} & \textbf{width, own} & \textbf{width, unit-level} \\
\midrule
Hill-A & CausalPFN & 35.13 & 0.0771 & 0.07 & --- & 0.055 \\
 & CNP head & 2.83 & 0.2367 & 0.06 & 0.092 & 0.012 \\
 & latent NP head & 24.19 & 0.1613 & 0.06 & 0.103 & 0.038 \\
\addlinespace[2pt]
Hill-B & CausalPFN & 1.42 & 0.0038 & 0.28 & --- & 0.047 \\
 & CNP head & 0.84 & 0.0183 & 0.52 & 0.184 & 0.050 \\
 & latent NP head & 0.97 & 0.0122 & 0.52 & 0.198 & 0.068 \\
\addlinespace[2pt]
Friedman & CausalPFN & 11.75 & 0.0334 & 0.10 & --- & 0.041 \\
 & CNP head & 1.28 & 0.0618 & 0.07 & 0.076 & 0.010 \\
 & latent NP head & 13.19 & 0.0708 & 0.07 & 0.067 & 0.033 \\
\addlinespace[2pt]
Kang--Schafer & CausalPFN & 1.43 & 0.0185 & 0.09 & --- & 0.013 \\
 & CNP head & 2.85 & 0.0393 & 0.05 & 0.069 & 0.010 \\
 & latent NP head & 13.65 & 0.0382 & 0.05 & 0.050 & 0.026 \\
\addlinespace[2pt]
IHDP & CausalPFN & 5.43 & 0.0420 & 0.11 & --- & 0.055 \\
 & CNP head & 3.90 & 0.0554 & 0.20 & 0.192 & 0.040 \\
 & latent NP head & 8.84 & 0.0942 & 0.21 & 0.221 & 0.084 \\
\addlinespace[2pt]
ACIC 2016 & CausalPFN & 1.63 & 0.0202 & 0.27 & --- & 0.019 \\
 & CNP head & 0.96 & 0.0261 & 0.30 & 0.093 & 0.016 \\
 & latent NP head & 2.46 & 0.0181 & 0.30 & 0.111 & 0.029 \\
\bottomrule
\end{tabular}
\caption{Unit-level calibration, error correlation, accuracy, and interval width on the six benchmarks. $m_S$ and $\rho_S$ use the estimator of Appendix~\ref{app:protocol}, from each model's unit-level means and standard deviations (for the CNP head, the head queried with each unit alone), and nRMSE is defined as in Table~\ref{tab:legacy}. CausalPFN's $m_S$, $\rho_S$, and nRMSE are those of Table~\ref{tab:legacy}. Widths are mean widths of the full-sample interval in units of $\mathrm{sd}(y)$: \emph{own} is the head's SCM-marginalized interval, and \emph{unit-level} is unit-level aggregation from the same model. CausalPFN's width is computed with the whole sample as context, at the temperatures fitted for Table~\ref{tab:legacy}. Heads: mean over three seeds.}
\label{tab:heads-moments}
\end{table}

\begin{figure}[tp]
\centering
\IfFileExists{figures/cfm_fixed_weight_cnp.pdf}{\includegraphics[width=\textwidth]{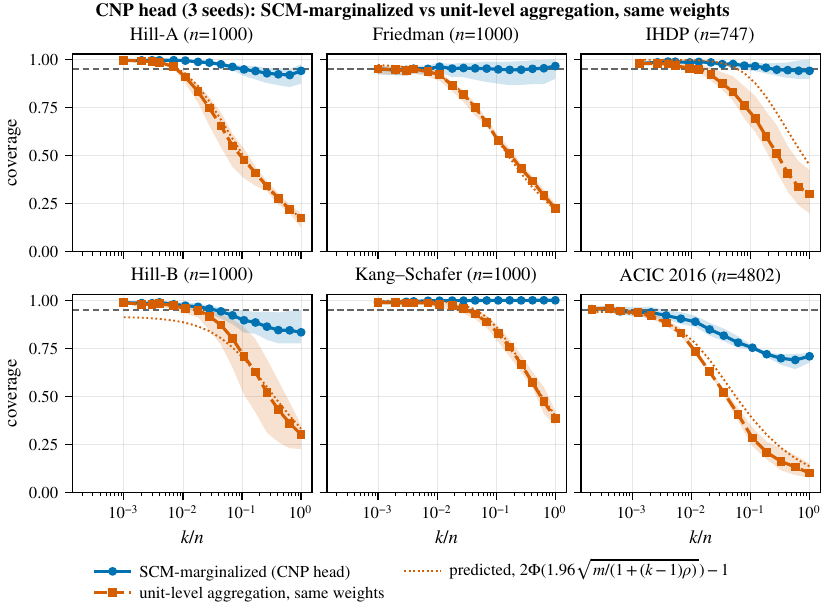}}{\fbox{\parbox{0.9\textwidth}{\centering\ttfamily MISSING FIGURE: figures/cfm\_fixed\_weight\_cnp.pdf}}}
\caption{Fixed-weight comparison for the CNP head on the six benchmarks, with subsets drawn uniformly at random: SCM-marginalized aggregation and unit-level aggregation from the same weights. Lines show the mean over three seeds and bands their range. Dotted lines average the three seeds' predictions from Equation~\eqref{eq:intro-coverage} for unit-level aggregation, each with $m_S$ and $\rho_S$ measured once per benchmark.}
\label{fig:heads-cnp}
\end{figure}

\begin{figure}[tp]
\centering
\IfFileExists{figures/cfm_fixed_weight_latent.pdf}{\includegraphics[width=\textwidth]{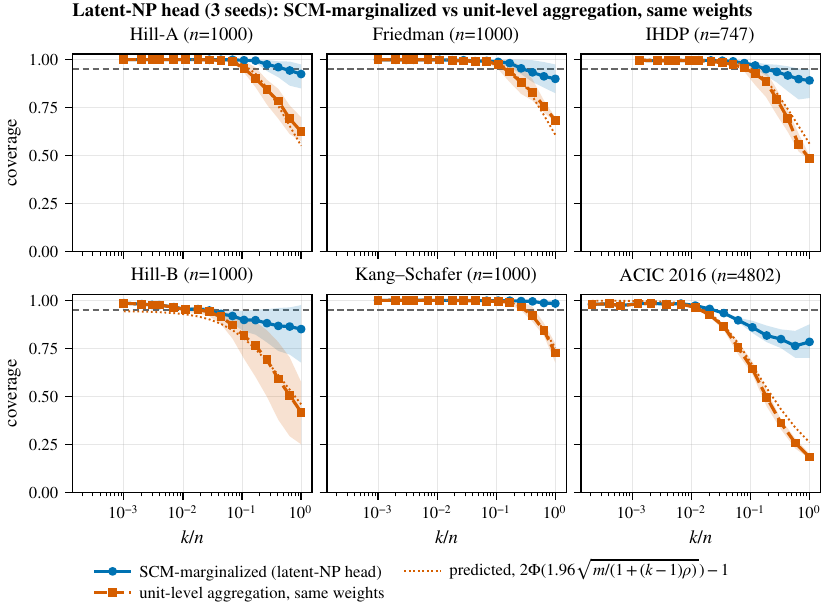}}{\fbox{\parbox{0.9\textwidth}{\centering\ttfamily MISSING FIGURE: figures/cfm\_fixed\_weight\_latent.pdf}}}
\caption{As Figure~\ref{fig:heads-cnp}, for the latent NP head.}
\label{fig:heads-latent}
\end{figure}

\paragraph{Heads on the benchmarks.} Tables~\ref{tab:heads-coverage} and~\ref{tab:heads-moments} and Figures~\ref{fig:heads-cnp} and~\ref{fig:heads-latent} give the results behind Section~\ref{sec:heads}. The per-seed range of SCM-marginalized coverage at $k=n$ is widest for the latent NP head on Hill-B ($0.68$ to $0.97$) and ACIC 2016 ($0.70$ to $0.88$). Equation~\eqref{eq:intro-coverage} follows unit-level coverage closely on Hill-A, Friedman, and Kang--Schafer, with a largest error of at most $0.09$ for either head. Its largest errors are on IHDP ($0.23$ for the CNP head and $0.15$ for the latent NP head) and, for the CNP head, on Hill-B ($0.12$), the two benchmarks on which $m_S$ depends most on the estimator (Appendix~\ref{app:protocol}).

\paragraph{Context on ACIC 2016.} Figure~\ref{fig:benchmarks} and Table~\ref{tab:legacy} use CausalPFN's public interface, which on ACIC 2016 conditions each window of units on its nearest neighbours instead of the whole sample. Re-scored with the whole sample as context, as for the heads, and at the temperatures fitted for Figure~\ref{fig:benchmarks}, CausalPFN's full-sample coverage on ACIC 2016 is $0.25$ instead of $0.20$. On the other benchmarks the two agree to within $0.03$.

\appendixsection{Proofs}
\label{app:proofs}

Throughout, $\E$, $\Var$, $\Cov$, and $\Pr$ are over $\zeta\sim\pi$, $\mathcal{D}\mid\zeta$, and $S$, as in Equation~\eqref{eq:coverage-def}. All errors have finite second moments, and $b_i := \E[e_i \mid \mathcal{D}]$. We say that $S$ is \emph{selected on the data} if $S$ is conditionally independent of $(\tau(X_i;\zeta))_{i=1}^n$ given $\mathcal{D}$. Uniformly drawn subsets and subsets chosen by covariates or by the model's own estimates are selected on the data. In our synthetic family, subsets chosen by the true propensity score are as well, because $w_p$ is independent of $(w_b, w_e)$ under the prior and enters a separate factor of the likelihood, so the two remain independent given $\mathcal{D}$. Subsets chosen by the true effect are not.

\paragraph{Proof of Proposition~\ref{prop:design}.}
Since $\bar e_S = k^{-1}\sum_{i\in S}e_i$,
\[
k^2\,\bar e_S^2 \;=\; \sum_{i\in S}e_i^2 \;+\; \sum_{\substack{i,j\in S\\ i\neq j}}e_ie_j .
\]
Taking expectations and using Equation~\eqref{eq:moments} gives $k^2\E[\bar e_S^2] = k\sigma_S^2 + k(k-1)c_S$, which is Equation~\eqref{eq:design}. By Equation~\eqref{eq:sigma}, $\E[\hat\sigma_S^2] = k^{-1}\E[\bar{\hat\sigma}_S^2] = k^{-1}m_S\sigma_S^2$, and dividing by Equation~\eqref{eq:design} gives Equation~\eqref{eq:bias}. If $\hat\sigma_S^2$ takes the same value $v_k$ for every query of size $k$, then $v_k = \E[\hat\sigma_S^2] = R\,\E[\bar e_S^2]$, and $\theta_S\in\hat\theta_S\pm z_{1-\alpha/2}\sqrt{v_k}$ if and only if $|\bar e_S|\le z_{1-\alpha/2}\sqrt{R\,\E[\bar e_S^2]}$, which is Equation~\eqref{eq:coverage}. If in addition $\bar e_S$ is mean-zero Gaussian, the standardized error is standard normal and $\mathcal{C}(k) = 2\Phi(z_{1-\alpha/2}\sqrt{R})-1$. \qed

The same expansion, with $\E[(\sum_{i\in S}e_i)^2]\ge 0$ and $(\sum_{i\in S}e_i)^2\le k\sum_{i\in S}e_i^2$, gives $-1/(k-1)\le\rho_S\le 1$ for $k\ge 2$.

\paragraph{Coverage when the reported variance varies.}
Suppose that, given $(\mathcal{D},S)$, $\bar e_S$ is Gaussian with mean zero and variance $v_{\mathcal{D},S}$, as for the posterior mean when $S$ is selected on the data and the posterior of $\theta_S$ is Gaussian. Then $\Pr(\theta_S\in[\hat l_S,\hat u_S]\mid\mathcal{D},S) = f(R_{\mathcal{D},S})$ with $f(r) = 2\Phi(z_{1-\alpha/2}\sqrt{r})-1$ and $R_{\mathcal{D},S} = \hat\sigma_S^2/v_{\mathcal{D},S}$, so $\mathcal{C}(k) = \E[f(R_{\mathcal{D},S})]$. The ratio $R$ of Equation~\eqref{eq:bias} is the average of $R_{\mathcal{D},S}$ weighted by $v_{\mathcal{D},S}/\E[v_{\mathcal{D},S}]$. Since $f$ is concave, $\mathcal{C}(k)\le f(R)$ when $R_{\mathcal{D},S}$ is independent of $v_{\mathcal{D},S}$, and Equation~\eqref{eq:intro-coverage} then overstates coverage. Without this independence the direction can reverse. If $\hat\sigma_S$ is fixed and only $v_{\mathcal{D},S}$ varies, coverage is convex in $v_{\mathcal{D},S}$ wherever it is below $2\Phi(\sqrt{3})-1\approx 0.92$, and pooling then raises it above $f(R)$.

\paragraph{Proof of Proposition~\ref{prop:source}.}
Given $\mathcal{D}$, each $\hat\tau_i$ and each $X_i$ is fixed, and $\tau_i = \tau(X_i;\zeta)$ with $\zeta\sim\pi(\zeta\mid\mathcal{D})$. Since $e_i$ differs from $-\tau_i$ by a quantity that is fixed given $\mathcal{D}$, $\Cov(e_i,e_j\mid\mathcal{D}) = \Cov(\tau_i,\tau_j\mid\mathcal{D})$, which is Equation~\eqref{eq:cS-definetti}. If $S$ is selected on the data, the law of the unit effects given $(\mathcal{D},S)$ equals their law given $\mathcal{D}$, so for $i,j\in S$,
\[
\E[e_ie_j\mid\mathcal{D},S] \;=\; \Cov_{\zeta\sim\pi(\zeta\mid\mathcal{D})}\big(\tau(X_i;\zeta),\tau(X_j;\zeta)\big) + b_ib_j .
\]
Averaging over pairs of distinct units in $S$ and then over $(\mathcal{D},S)$ gives $c_S$. The posterior mean under $\pi$ has $b_i = 0$. \qed

\paragraph{Aggregation for a general model.} For a model with joint predictive distribution $Q_S(\cdot\mid\mathcal{D})$ for $(\tau_i)_{i\in S}$ and marginals $Q_i(\cdot\mid\mathcal{D})$, unit-level aggregation reports the law of $k^{-1}\sum_{i\in S}\tilde\tau_i$ under $(\tilde\tau_i)_{i\in S}\sim\bigotimes_{i\in S}Q_i$, and SCM-marginalized aggregation reports its law under $(\tilde\tau_i)_{i\in S}\sim Q_S$. For the exact posterior, $Q_S$ is the law of $(\tau(X_i;\zeta))_{i\in S}$ under $\zeta\sim\pi(\zeta\mid\mathcal{D})$, and the two are Equations~\eqref{eq:intro-unit} and~\eqref{eq:intro-scm} with $p(\tau_i\mid\zeta)$ the point mass at $\tau(X_i;\zeta)$. For the latent NP, $Q_S$ is the distribution in Equation~\eqref{eq:latentnp}. The CNP outputs a Gaussian for $\theta_S$ directly (Proposition~\ref{prop:cnp}), and a model that returns only marginals, such as CausalPFN, supports only unit-level aggregation.

\begin{corollary}[Exact posterior]
\label{cor:exact}
Let $\hat\tau_i = \E[\tau_i\mid\mathcal{D}]$ and $\hat\sigma_i^2 = \Var(\tau_i\mid\mathcal{D})$ under $\pi$, and let $S$ be selected on the data. Then (i) unit-level aggregation has $m_S = 1$ and $R = 1/(1+(k-1)\rho_S)$; (ii) SCM-marginalized aggregation, which reports $\Var(\theta_S\mid\mathcal{D})$, has $R = 1$ with this variance in place of $\hat\sigma_S^2$; and (iii) if the posterior of $\theta_S$ has a continuous distribution function, its equal-tailed $1-\alpha$ credible interval has $\mathcal{C}(k) = 1-\alpha$ for every $k$.
\end{corollary}
\begin{proof}
(i) For $i\in S$, $\E[e_i^2\mid\mathcal{D},S] = \Var(\tau_i\mid\mathcal{D}) = \hat\sigma_i^2$, so $\E[\bar{\hat\sigma}_S^2] = \sigma_S^2$ and $m_S = 1$, and Proposition~\ref{prop:design} gives $R$. (ii) The point estimate is $\E[\theta_S\mid\mathcal{D}]$ and the reported variance is $\Var(\theta_S\mid\mathcal{D})$, so $\E[\bar e_S^2] = \E[\Var(\theta_S\mid\mathcal{D})]$, which equals the expected reported variance. (iii) Given $(\mathcal{D},S)$, $\theta_S$ follows its posterior, so the interval contains it with conditional probability $1-\alpha$. Averaging over $(\mathcal{D},S)$ gives the claim.
\end{proof}
Part (iii) concerns coverage averaged over $\pi$, as in Equation~\eqref{eq:coverage-def}, and does not imply coverage at a fixed SCM. In our synthetic family the outcome is linear in Gaussian coefficients with known noise variance, and $w_p$ enters a separate factor of the likelihood, so the posterior of $(\tau_i)_{i\in S}$ is Gaussian and the interval in (iii) is $\hat\theta_S\pm z_{1-\alpha/2}\Var(\theta_S\mid\mathcal{D})^{1/2}$. Because $s$ enters only the propensity factor, this posterior is the same function of $\mathcal{D}$ at every $s$, so the corollary applies at $s=4$ as well, and the population minimizer of Proposition~\ref{prop:cnp} at $s=1$ is also the exact posterior at $s=4$.

\paragraph{Marginals do not determine the joint.}
Let the unit effects have marginals $\mathcal{N}(\mu_i,\omega^2)$, $i\in S$. For every $\rho\in[0,1]$, the Gaussian with these marginals and all pairwise correlations $\rho$ gives $\Var(k^{-1}\sum_{i\in S}\tau_i) = \omega^2\{1+(k-1)\rho\}/k$. This ranges from $\omega^2/k$ under independence to $\omega^2$ at $\rho=1$, the largest value any joint distribution with these marginals allows, since the standard deviation of a sum is at most the sum of the standard deviations. The marginals are thus consistent with every design effect between $1$ and $k$, and no function of them alone equals $\Var(\theta_S\mid\mathcal{D})$ for every joint posterior. A rule that also uses the model's prior can sometimes do better. In a linear-Gaussian model with known features, for instance, the marginal variances at enough units determine the posterior covariance of the coefficients. The statement in Section~\ref{sec:definetti} concerns rules that use the marginals alone.

\begin{proposition}[Bayes-optimal Gaussian output]
\label{prop:cnp}
Let training tasks draw $\zeta\sim\pi_0$, $\mathcal{D}\mid\zeta$, and a subset $S$ selected on the data, with $\E_{\pi_0}[\theta_S^2]<\infty$ and $\E_{\pi_0}\lvert\log\Var_{\pi_0}(\theta_S\mid\mathcal{D})\rvert<\infty$. Among all measurable maps $(\mathcal{D},S)\mapsto(\mu,v)$, the expected loss $\E_{\pi_0}\big[-\log\mathcal{N}\big(\theta_S;\mu(\mathcal{D},S),v(\mathcal{D},S)\big)\big]$ is minimized by $\mu = \E_{\pi_0}[\theta_S\mid\mathcal{D}]$ and $v = \Var_{\pi_0}(\theta_S\mid\mathcal{D})$, uniquely up to sets of probability zero under the training distribution of $(\mathcal{D},S)$.
\end{proposition}
\begin{proof}
Given $(\mathcal{D},S)$, the conditional expected loss equals $\tfrac12\log v + \{\Var(\theta_S\mid\mathcal{D},S) + (\E[\theta_S\mid\mathcal{D},S]-\mu)^2\}/(2v)$ plus a constant, which is uniquely minimized at $\mu = \E[\theta_S\mid\mathcal{D},S]$ and $v = \Var(\theta_S\mid\mathcal{D},S)$. Because $S$ is selected on the data, these equal the posterior mean and variance of $\theta_S$ given $\mathcal{D}$. Minimizing the conditional loss for each $(\mathcal{D},S)$ minimizes its expectation.
\end{proof}
As in \citet[Insight~1]{muller2022}, the training loss is an expected cross-entropy to the posterior, and restricted to Gaussian outputs its minimizer matches the posterior's first two moments. The result describes the population minimizer. A trained network approximates it well only for tasks like those seen in training, which for our CNP include subsets of every size but not datasets at $s=4$.

\end{document}